\documentclass{article}

\usepackage{arxiv}

\usepackage[utf8]{inputenc} 
\usepackage[T1]{fontenc}    
\usepackage{hyperref}       
\usepackage{url}            
\usepackage{booktabs}       
\usepackage{amsfonts}       
\usepackage{nicefrac}       
\usepackage{microtype}      
\usepackage{lipsum}
\usepackage{amsmath,amssymb,amsfonts}
\usepackage{textcomp}
\usepackage{xcolor}
\usepackage{graphicx}
\usepackage{multirow}
\usepackage{tabularx}
\usepackage{algorithm}
\usepackage{algorithmic}
\newtheorem{lemma}{Lemma}
\newtheorem{definition}{Definition}
\newtheorem{theorem}{Theorem}
\newtheorem{proof}{Proof}
\usepackage[tight,footnotesize]{subfigure}
\usepackage{url}
\usepackage{doi}
\title{PPGAN: Privacy-preserving Generative Adversarial Network}

\author{
  Yi Liu\\
  School of Data Science and Technology\\
  Heilongjiang University\\
   \And
 Jialiang Peng\thanks{*Corresponding Author: Pengjialiang@hlju.edu.cn. This work is supported by the Ministry of Education of China and the School of Entrepreneurship Education of Heilongjiang University (Grant NO.201910212133) and Heilongjiang Provincial Natural Science Foundation of China (Grant NO.QC2016091).}  \\
  School of Data Science and Technology\\
  Heilongjiang University\\
  \And 
  James J.Q. Yu\\
  Department of Computer Science and Engineering\\
  Southern University of Science and Technology\\
  \And 
   Yi Wu\\
  School of Data Science and Technology\\
  Heilongjiang University\\
}

\begin{document}
\maketitle

\begin{abstract}
Generative Adversarial Network (GAN) and its variants serve as a perfect representation of the data generation model, providing researchers with a large amount of high-quality generated data. They illustrate a promising direction for research with limited data availability. When GAN learns the semantic-rich data distribution from a dataset, the density of the generated distribution tends to concentrate on the training data. Due to the gradient parameters of the deep neural network contain the data distribution of the training samples, they can easily \textit{remember} the training samples. When GAN is applied to private or sensitive data, for instance, patient medical records, as private information may be leakage. To address this issue, we propose a \textit{Privacy-preserving Generative Adversarial Network (PPGAN)} model, in which we achieve differential privacy in GANs by adding well-designed noise to the gradient during the model learning procedure. Besides, we introduced the \textit{Moments Accountant} strategy in the PPGAN training process to improve the stability and compatibility of the model by controlling privacy loss. We also give a mathematical proof of the differential privacy discriminator. Through extensive case studies of the benchmark datasets, we demonstrate that PPGAN can generate high-quality synthetic data while retaining the required data available under a reasonable privacy budget.
 
\end{abstract}


\section{INTRODUCTION}
In recent years, researchers have used a large number of training data to perform data mining tasks, in the field of medical and health informatics, such as disease prediction and auxiliary diagnosis \cite{ref-33}. Deep learning models are employed to remember the characteristics of a large number of training samples for classification or prediction purposes. However, organizations such as hospitals and research institutes are paying more and more attention to the protection of data. Additionally, the \textit{General Data Protection Regulation (GDPR)}\cite{ref-1} issued by the European Union prohibits organizations from sharing private data. It is increasingly difficult for researchers to obtain training data unlimited legally.  

Fortunately, the generative model provides us with a solution to the issue of data scarcity \cite{ref-18}, yet data privacy leakage issues may arise. StyleGAN \cite{ref-2} shown impressive performance in generating fake face images. In principle, it can memorize data distribution from the small amount of training data, rendering indistinguishable high-quality ``fake" samples. However, for most people, they expect their face data not to be used as a training sample. 

GAN can implicitly disclose the privacy information of training samples. GAN model produces high-quality "fake" samples through continuous training and resampling. This training method grants hackers the opportunity to restore the original samples. Therefore, we not only need high-quality sample generation approaches but also need to achieve a reasonable level of data privacy.

Based on the above findings, we propose a \textbf{Privacy-preserving GAN (PPGAN)}. PPGAN combines with differential privacy \cite{ref-5} to ensure that the exact training samples can not be revealed by adversaries from the trained model, resulting in well-protected data privacy. In particular, we added well-designed noise to the gradients in the training process in PPGAN and used the framework of the WGAN \cite{ref-3} model as the main skeleton of PPGAN. The proposed model does not suffer from a privacy leakage issue whose proportional to the volume of data thanks to the introduced average aggregator that offsets the privacy overhead of large datasets.

We would like to point out our main contributions as follows:
\begin{itemize}
	\item [$\bullet$] We propose the PPGAN framework that can generate high-quality data points while protecting data privacy. PPGAN combines noise well-designed in the differential privacy with training gradients to disturb the distribution of the original data. Finally, we give a rigorous proof of the differential privacy discriminator in mathematics.
		
	\item [$\bullet$] We introduced the \textit{Moments Accountant} strategy that maintains the boundedness of the function, controls the privacy level and significantly improves the stability of the model training.
	
	\item [$\bullet$]  We evaluated PPGAN with benchmark datasets. The results show that PPGAN can generate high-quality data with adequately protected privacy under a reasonable privacy budget.
	
\end{itemize}
The overall structure of this paper is as follows. First, we briefly summarize the relevant literature in Section \ref {sec2} and then introduce the proposed PPGAN framework and its theoretical proof in Section \ref {sec-3}. We assess the performance of our framework in Section \ref{sec-4}. Finally, this paper is concluded in Section \ref{sec-5}.

\section{RELATED WORK}\label{sec2}
In this section, we focus on the literature on privacy-preserving deep learning. Existing literature can be roughly classified along several axes: generative adversarial networks in the medical field and privacy-preserving deep learning.

\textbf{Generative Adversarial Network.} In recent years, GAN and its variants have made meaningful progress in the academic and medical fields. Choi et al. \cite{ref-11} proposed medGAN, which is a generative adversarial network for generating multi-label discrete patient records. Brett K. Beaulieu-Jones et al. \cite{ref-10} proposed AC-GAN (under differential privacy and labeled private) to simulate participants in the \textit{SPRINT} clinical trial. However, the previously described GANs do not meet the data management requirements of \textit{GDPR} for privacy data protection.

\textbf{Privacy-Preserving Deep Learning.} Differential privacy (DP), local differential privacy (LDP), and other related algorithms combined with deep neural networks have become one of the most popular algorithmic models in the field of privacy protection. Dwork et al. \cite{ref-4}, the author of the concept of differential privacy, laid a lot of theoretical foundations for the field of differential privacy. Song et al. \cite{ref-12} added perturbations to random descent gradients, which can improve network performance after batch training. Many machine learning algorithms can achieve differential private by introducing randomization in the calculation, usually by noise \cite{ref-12}. 

We propose \textit{PPGAN} to address the challenges that appeared in the previous works. In \cite{ref-15}, although the privacy-preserving deep learning system does not need to share datasets, it still reveals the user's privacy when uploading local parameters to the server. What is different from \cite{ref-16} is that we add well-designed noise during the process of stochastic gradient descent. We introduced a moments accountant strategy, which not only successfully incorporated the privacy enhancement mechanism into the training depth generation model but also significantly improved the stability and scalability of the generation model training itself \cite{ref-21}.

\section{METHODOLOGY}\label{sec-3}
In this section, we elaborate on the proposed privacy protection framework PPGAN. We first introduce the concept of differential privacy. Subsequently, a brief introduction to GAN and WGAN. After that, we show the proposed PPGAN with theoretical analyses and the way noise is added to the gradients. Finally, we introduce \emph{moments accountant} \cite{ref-14}, which is the fundamental idea in our framework to ensure the privacy of the iterative gradient descent process. We strictly prove in mathematics that the use of the moments accountant allows the discriminator to guarantee differential privacy.  

\subsection{Differential Privacy}
Differential privacy (DP) \cite{ref-4,ref-5,ref-14} constitutes a solid standard for privacy guarantee for algorithms on the  database. For all two datasets $x$ and $y$, which differ by at most one record, we refer to these two datasets as a neighboring dataset. In the above description, natural measure of the distance between two databases $x$ and $y$ will be their  distance:
\begin{definition}{\textbf{(Distance Between Databases)}}\label{defi-8}\\
	 The ${\ell _1}$ norm of a database $x$ is denoted $||x|{|_1}$ and is defined to be: 
	\begin{equation}
	||x|{|_1} = \sum\limits_{i = 1}^{|\aleph |} {|{x_i}|} 
	\end{equation}
	The ${\ell _1}$ distance between two databases $x$ and $y$ is $||x - y|{|_1}$. In particular, when $||x - y|{|_1} = 1$, $x$ and $y$ are mutually referred to as neighboring datasets.
\end{definition}
\begin{definition}{\textbf{($(\varepsilon ,\delta )$-DP)}}\label{defi-1}\\
	A randomized algorithm $\phi ( \cdot )$ with domain ${\Phi ^{|\chi |}}$ is $(\varepsilon ,\delta )$-DP if for all $O \subseteq Range(\phi)$ and for all $d,d' \in {{\rm \Phi}^{|\aleph |}}$ (for any neighbouring datasets) such that $||d - d'|| \le 1$ :
	\begin{equation}
		Pr [\phi(d) \in O] \le {e^\varepsilon }Pr [\phi(d') \in O] + \delta \
	\end{equation}
\end{definition}
Noted that $\epsilon$ stands for privacy budget, which controls the level of privacy guarantee achieved by mechanism $\phi$. And when $\varepsilon  = \infty $, this case is non-private. 

Among the mechanisms for achieving differential privacy, the two most widely used are the \textit{Laplace mechanism and the Gaussian noise mechanism (GNM)} \cite{ref-31}. Due to the combined properties of the GNM, it is prevalent in many DP protection models. In PPGAN, we use the GNM because the moments accountant (detailed in Section \ref{sec-3-3}) provides an improved privacy boundary analysis and is well-matched to the combined properties of the GNM. The GNM is defined as follows:
\begin{equation}
\phi(x) \buildrel \Delta \over = f(x) + N(0,{\sigma ^2}{s_f}^2)
\end{equation}
where $s_f$ is defined as \textit{sensitivity}, which is only related to query type $f$. The sensitivity is defined as follows:
\begin{definition}\label{defi-sens}
	\textbf{($L_2$ norm-Sensitivity)}\\
	We given the neighboring datasets $x$ and $x'$ and given a query $f:x \to \Omega $, the \textit{sensitivity} of $f$ as follows:
	\begin{equation}
	\Delta f = \mathop {\max }\limits_{x,x'} ||f(x) - f(x')|{|_2}
	\end{equation}
\end{definition}
Noted that it records the largest difference between query results on datasets $x$ and $x'$.

According to the algorithm $\phi( \cdot )$ in Definition \ref{defi-1} is stochastic and is not related to the distribution of the output data. Moreover, the Gaussian noise mechanism adds a well-design noise to a single gradient without affecting the entire gradient aggregation. Therefore, we can use this attribute with GAN so that GAN can generate high-quality data while satisfying differential privacy.

\subsection{GAN and WGAN}
Generative adversarial network (GAN) \cite{ref-29,ref-30} is a class of deep neural network architectures comprised of two networks, pitting one against the other (thus the \textit{``adversarial"}). 
Suppose our generative model is $G(z)$, where $z$ is random noise and $G$ converts this random noise into  $x$. Take with contradicting training adjective \textit{Electronic Health Record (EHR)} as an example. Let $G$ be a generator synthesizing EHR, and $D$ is the discriminator in the generator model. For an arbitrary input $x$, the output of $D(x)$ is a real number in the range [0,1] that determines how likely this EHR is authentic. Let $P_r$ and $P_g$ represent the distribution of real ones and the distribution of generated EHRs, respectively. The objective function of the discriminative model is as follows:
\begin{equation}
\mathop {\max }\limits_D {E_{x \sim \Pr }}[\log (D(x)] + {E_{x \sim Pg}}[\log (1 - D(x)]
\end{equation}

The goal of a similar from distinguishing is to prevent them from real records and the generated ones. The entire optimization objective function is as follows:
\begin{equation}
\begin{aligned}
\mathop {\min }\limits_G \mathop {\max }\limits_D V(G,D)& = {E_{x \sim {P_{data}}}}_{(x)}[\log (D(x)] \\&+ {E_{z \sim {P_z}(z)}}[\log (1 - D(G(z)))]
\end{aligned}
\end{equation}

WGAN \cite{ref-3} uses the \textit{Wasserstein} distance instead of the \textit{Jensen-Shannon} distance. Compared with the original GAN, WGAN's parameters are less sensitive and the training process is smoother. It solves a minimax two-player game that finds the balance point of each other:
\begin{equation}
\mathop {\min }\limits_G \mathop {\max }\limits_{w \in W} {E_{x \sim {\mathop{\rm P}\nolimits} data(x)}}[{f_w}(x)] - {E_{z \sim {P_{z(z)}}}}[{f_w}(G(z))]
\end{equation}

\subsection{PPGAN framework}
In this section, we present the proposed \textit{Privacy-preserving Generative Adversarial Network (PPGAN) model}, which is detailed in Algorithm \ref{al-2} and illustrated in Fig. \ref{fig3}. Noted that the $Discriminator$ has access to the real data, while the $Generator$ only receives feedback on the real data through the $Discriminator$'s output. This will be useful in PPGAN since only the $Discriminator$ is required to differential privacy. The $Generator$'s utilizes the result from the $Discriminator$, thus differential privacy \cite{ref-17}. (So we add noise proportional to the training data on the gradient of the Wasserstein distance, rather than adding noise to the final parameters.) 

\begin{figure*}[htbp] \centering 
	\includegraphics[width=1\textwidth]{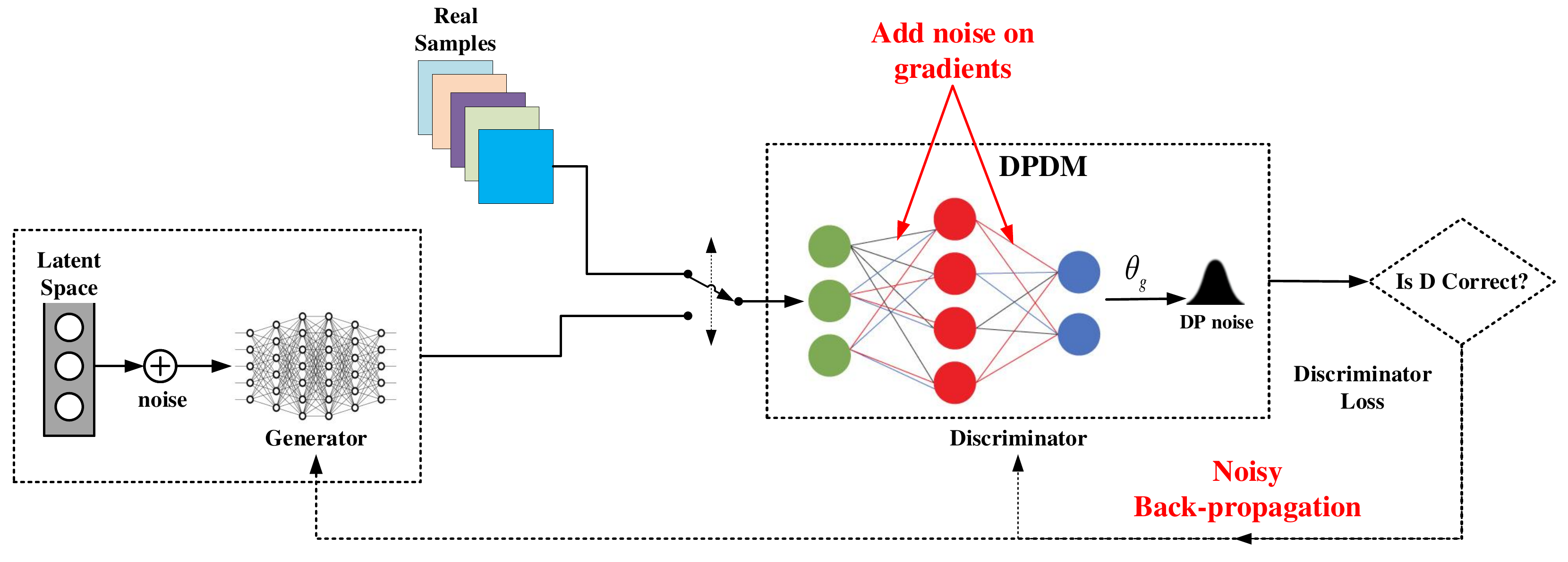}
	\caption{Overview of our Privacy-preserving Generative Adversarial Network (PPGAN) framework. } \label{fig3}
\end{figure*}

\begin{algorithm}
	\caption{Privacy-preserving Generative Adversarial Network (PPGAN)}
	\label{al-2}     
	\begin{algorithmic}[1]       
		\REQUIRE ~~\\     
		The learning rate: $\alpha $. The clipping parameter: $c$. The mini-batch size:
		$m$. The number of discriminator iterations per generator iteration: $n_d$. Generator iteration: $n_g$. Noise scale: ${\sigma _n}$. 
		\ENSURE ~~\\     
		DP generator $\theta $;
		\STATE	Initialize generator parameters and discriminator parameters ${\omega _0},{\theta _0}$, respectively.
		\FOR  {${t_1} = 1,...,{n_g}$}
		\FOR  {${t_2} = 1,...,{n_d}$}
		\STATE $\{ {x^{(i)}}\} _{i = 1}^m \sim {{\rm P}_\Upsilon }$ a mini-batch from the real data.\	
		\STATE $\{ {z^{(i)}}\} _{i = 1}^m \sim p(z)$ a mini-batch of prior samples.\label{line5}
		\STATE ${g_\omega} \leftarrow {g_\omega}\min (1,C/||g_\omega||) + N(0,{\sigma _n}^2c_g^2I)$ \textbf{(adding noise)}
		\STATE $\omega  \leftarrow clip(\omega  + \alpha  \cdot SGD(\omega ,{g_\omega }), - c,c)$
		\ENDFOR
		\STATE ${g_\delta} \leftarrow {g_\delta}\min (1,C/||g_\delta||)$
		\STATE $\theta  \leftarrow \theta  - \alpha  \cdot SGD(\theta ,{g_\theta })$
		\ENDFOR
		\RETURN $\theta$;      
	\end{algorithmic}
\end{algorithm}
\subsection{Privacy Guarantees of PPGAN}\label{sec-3-3}

To show that PPGAN in Algorithm \ref{al-2} does satisfy the differential privacy, we prove that the parameters of the generator guarantee the differential privacy relative to the sample training point under the condition that the discriminator parameters satisfy the differential privacy. Therefore, the generated data from $G$ satisfies the differential privacy, which means that $G$ does not leakage the privacy of the dataset \cite{ref-32}. Through moment accountant strategy, we can control the boundary of ${g_w}({x^{(i)}},{z^{(i)}})$ and calculate the final privacy loss. Along with Definition \ref{defi-1}, intuitively, we have the definition of privacy loss at $\tau$:
\begin{definition}\textbf{(Privacy Loss)} \label{defi-3}\\
	\begin{equation}
	c(\tau ;\phi ,aux,d,d') \buildrel \Delta \over = \log \frac{{{\rm P}[\phi (aux,d) = \tau ]}}{{{\rm P}[\phi (aux,d') = \tau ]}}
	\end{equation}
\end{definition}
We introduce privacy loss to measure the distribution difference between two changing data. The privacy loss random variable is derived from the Definition \ref{defi-1}, which is used to describe the privacy budget of $\phi (d)$. For a given mechanism $\phi$, we define the ${\upsilon ^{th}}$ moment ${\beta _\phi}(\upsilon ;aux,d,d')$ as the log of the moment generating function evaluated at the value:
\begin{definition}\textbf{(Log moment generating function)}\label{defi-4}	
	\begin{equation}
	{\beta _\phi }(\upsilon ;aux,d,d') \buildrel \Delta \over = \log {E_{o \sim \phi }}[{e^{\upsilon C(\phi ,aux,d,d')}}]
	\end{equation}
\end{definition}
\begin{definition}\textbf{(Moments Accountant)}\label{defi-5}
	\begin{equation}
	{\beta _\phi }(\upsilon ) \buildrel \Delta \over = \mathop {\max }\limits_{aux,d,d'} {\beta _\phi }(\upsilon ;aux,d,d')
	\end{equation}
\end{definition}
The basic idea behind the moments accountant is to accumulate the privacy expenditure by framing the privacy loss as a random variable and using its moment-generating functions to understand that variable’s distribution better.  This property makes the PPGAN model training more stable \cite{ref-20}. The tail bound can also be applied to privacy guarantee (In \cite{ref-14}). Since the moments accountant saves a factor of $\sqrt {\log ({n_g}/\delta )} $, according to Definition \ref{defi-1}, this is a significant improvement for the large iteration ${{n_g}}$.

The following theorem, a proof of which can be found in \cite{ref-18,ref-14,ref-6,ref-19}, allows us to move the burden of differential privacy to the discriminator; the differential privacy of the generator will follow by the theorem.

\begin{theorem}\label{th-1}
	\textbf{(Post-processing)}\\ Let $\phi$ be an $(\varepsilon ,\delta )$-differentially private algorithm and let $f:\xi  \to \xi '$ where $\xi '$ is any arbitrary space. Then $f \circ \phi$ meets $(\varepsilon ,\delta )$-differentially private.
\end{theorem}

Next, we present the mathematical reasoning proof that the discriminator satisfies the differential privacy. First, we propose a lemma that PPGAN satisfies the definition of DP.

\begin{lemma}\label{le-1}
	Under the definition of GNM and $L_2$-sensitivity (in Definition \ref{defi-sens}), for any $\delta  \in (0,1)$, $\sigma  > \frac{{\sqrt {2\ln (1.25/\delta )} \Delta f}}{\varepsilon }$ , we have noise $Y \sim N(0,{\sigma ^2})$ satisfies $(\varepsilon ,\delta )$-DP.
\end{lemma}

\begin{proof}
	We assume that $\Delta f$ is the $L_2$-sensitivity, and according to the Definition. \ref{defi-1}, then we have:
	\begin{equation}
	|\ln \frac{{{e^{ - \frac{1}{{2{\sigma ^2}}}{x^2}}}}}{{{e^{ - \frac{1}{{2{\sigma ^2}}}{{(x + \Delta f)}^2}}}}}| = |\frac{1}{{2{\sigma ^2}}}(2x\Delta f + {(\Delta f)^2})| \le \varepsilon \\
	\therefore {\rm{ }}|x| \le \frac{{{\sigma ^2}\varepsilon }}{{\Delta f}} - \frac{{\Delta f}}{2}\Delta f .
	\end{equation}
	Let $t = \frac{{{\sigma ^2}\varepsilon }}{{\Delta f}} - \frac{{\Delta f}}{2}$, if and only if $||x|| \le t$, the distribution satisfies DP, and when $||x|| > t$, we want the probability of privacy leakage to be less than $\delta$, so we have:
	\begin{equation}
	P(x > t) < \frac{\delta }{2}
	\end{equation}
	where $P( \cdot )$ denotes the probability of revealing privacy. Next, we prove that the Gaussian distribution function is bounded above:
	\begin{equation}
	\begin{aligned}
P(x > t) &= \frac{1}{{\sqrt {2\pi \sigma } }}\int_t^\infty  {{e^{ - \frac{{{x^2}}}{{2{\sigma ^2}}}}}} dx < \frac{1}{{\sqrt {2\pi \sigma } }}\int_t^\infty  {\frac{x}{t}{e^{ - \frac{{{x^2}}}{{2{\sigma ^2}}}}}} dx \\&= \frac{\sigma }{{\sqrt {2\pi t} }}{e^{ - \frac{{{t^2}}}{{2{\sigma ^2}}}}}(\because x > t)\
    \end{aligned}
	\end{equation}
	Then the problem is converted to:
	\begin{equation}
	\begin{gathered}
	\frac{\sigma }{{\sqrt {2\pi t} }}{e^{ - \frac{{{t^2}}}{{2{\sigma ^2}}}}} < \frac{\delta }{2} \hfill ,
	\frac{t}{\sigma }{e^{\frac{{{t^2}}}{{2{\sigma ^2}}}}} > \frac{2}{{\sqrt {2\pi \sigma } }} \hfill ,
	\ln \frac{t}{\sigma } + \frac{{{t^2}}}{{2{\sigma ^2}}} > \ln \frac{2}{{\sqrt {2\pi \sigma } }} \hfill 
	\end{gathered} 
	\end{equation}
	\begin{equation}\label{eq-left}
	\therefore \left\{ \begin{gathered}
	\ln \frac{t}{\sigma } \geqslant 0 \hfill \\
	\frac{{{t^2}}}{{2{\sigma ^2}}} > \ln \frac{2}{{\sqrt {2\pi \delta } }} \hfill \\ 
	\end{gathered}  \right.
	\end{equation}
	For the left two terms of Equation \ref{eq-left}, because $t = \frac{{{\sigma ^2}\varepsilon }}{{\Delta f}} - \frac{{\Delta f}}{2}$, let $\sigma  = c\frac{{\Delta f}}{\varepsilon }$, then $t = c\sigma  - \frac{{\Delta f}}{2}$, thus we have:
	\begin{equation}\label{eq-t}
	\frac{t}{\sigma } = c - \frac{{\Delta f}}{{2\sigma }} = c - \frac{\varepsilon }{{2c}}.
	\end{equation}
	Here $\varepsilon  < 1,c \geqslant 1$, then
	\begin{equation}\label{eq-c}
	\ln (c - \frac{\varepsilon }{{2c}}) > \ln (c - \frac{1}{2}) \geqslant 0.
	\end{equation}
	By Equation \ref{eq-c}, we have $c \geqslant \frac{3}{2}$. By Equation \ref{eq-t}, we have:
	\begin{equation}
	\frac{{{t^2}}}{{2{\sigma ^2}}} = \frac{1}{2}({c^2} - \varepsilon  + \frac{{{\varepsilon ^2}}}{{4{c^2}}}).
	\end{equation}
	Because $\varepsilon  < 1,c \geqslant \frac{3}{2}$, we have:
	\begin{equation}
	{c^2} - \varepsilon  + \frac{{{\varepsilon ^2}}}{{4{c^2}}} > c^2 - \frac{8}{9} > 2\ln \frac{1}{{\sqrt {2\pi \delta } }}
	\end{equation}
	\begin{equation}
	\begin{gathered}
	{c^2} > \ln \frac{2}{\pi }{e^{\frac{8}{9}}} + 2\ln \frac{1}{\delta } \hfill ,
	\because \ln \frac{2}{\pi }{e^{\frac{8}{9}}} > {1.25^2} \hfill ,
	\therefore {c^2} > 2\ln \frac{{1.25}}{\delta } \hfill \\ 
	\end{gathered} 
	\end{equation}
	In the above equations, let $\sigma  = c\frac{{\Delta f}}{\varepsilon }$, so we have $\sigma  > \frac{{\sqrt {2\ln (1.25/\delta )} \Delta f}}{\varepsilon }$. In particular, in the SGD algorithm, Gaussian noise meets the definition of satisfying differential privacy as long as it satisfies $\sigma  \geqslant c\frac{{q\sqrt {T\ln (\frac{1}{\delta }} )}}{\varepsilon }$, where $q$ is the sampling probability and $T$ is the iteration round.$\blacksquare$
	
\end{proof} 

According to \cite{ref-18}, the conditions for the discriminator to guarantee differential privacy are given as follows:
\begin{equation}\label{eq-17}
{\sigma _n} = 2q\sqrt {{n_d}\log (\frac{1}{\delta }} )/\varepsilon 
\end{equation}
where $q$ is the sampling probability and $n_d$ is the number of iterations of the discriminator in each loop.
\begin{theorem}\label{th-2}
	Equation\ref{eq-17} represents the relationship between the noise level ${\sigma _n}$  and the privacy level $\epsilon$. When we give a fixed perturbation ${\sigma _n}$ on the gradient, according to Equation\ref{eq-17}, we know that the larger the $q$ is, the $D$ gets the fewer privacy guarantee. Because the $D$ calculates more data, the privacy that can be allocated on each data point is limited. In addition, due to the data provides more information,  more iterations ($n_d$) will result in fewer privacy guarantees. The facts described above require us to be cautious when choosing parameters to achieve a reasonable level of privacy.
\end{theorem}

PPGAN modifies the GAN framework to keep differentially private while relying on Theorem \ref{th-1},\ref{th-2} and Lemma \ref{le-1} to change the differential private $G$ to train the differentially private $D$. 

\section{EXPERIMENTS}\label{sec-4}
In this section, we will conduct a series of experiments to investigate how the privacy budget affects the effectiveness of PPGAN on the two benchmark datasets MNIST and MIMIC-III \cite{ref-med-1}. MIMIC-III is a well-known public EHR database that includes medical records of 46,520 intensive care units (ICUs) over the age of 11 \cite{ref-18}. We employ PPGAN to generate EHRs and protected privacy information at the same time. In the experiment, we focus on two issues: 1) Relationship between Privacy budget and Generation Performance; 2) Relationship between Privacy budget and High-quality Datasets.

\subsection{Data preprocessing}

First, we only use the extracted ICD9 code (The ICD9 code represents the type of disease, and the range of coding is $C \in [1,1071]$.) \cite{ref-27,ref-med-2} and use the first three digits for encoding. We then record the patient's admission to the disease and turn it into a vector $x$. For example, patient $P$ was diagnosed with three diseases at admission, and the disease codes are indicated by 9, 42, 146, respectively. (So the ICD9 code consists of 9, 42 and 146.) We use the vector $ x $ to indicate the patient's access record, where the vector is at position 9, the 42nd and 146th bits are set to 1, and the rest are set to 0. Then we aggregate the patient's longitudinal record into a single fixed-size vector $x \in {{\rm Z}^ + }$, where $|C| = 1071$ for dataset.

\subsection{Relationship between Privacy budget and Generation Performance}
In this section, we mainly explore the relationship between privacy budget and generation performance. Considering the combined properties data of $Gaussian$ $noise$, we add Gaussian noise in the process of stochastic gradient descent. Different Gaussian noises can produce different levels of privacy. We input the same set of MNIST image datasets and observe the output generated samples. In the experiments,  ${\alpha _d} = 5.0 \times {10^{ - 5}}$ learning rate of discriminator; ${\alpha _g} = 5.0 \times {10^{ - 5}}$, learning rate of generator; moments accountant parameter $C = 1.0 \times {10^{ - 2}}$; noise scale $\delta  = 1.0 \times {10^{ - 5}}$, and the number of iterations on discriminator $t_d$ and generator $t_g$ are 5 and $5.0 \times {10^5}$, respectively.
The experimental results are shown in Fig. \ref{fig4}. The code is available.\footnote{\url{https://github.com/hdliuyi/PPGANs-Privacy-preserving-GANs }}

As shown in Fig. \ref{fig4}, as the privacy budget increases, the quality of the generated images is getting worse. We add well-designed noise that disturbs the data point distribution of the image. Since the noise is randomly added, the distribution of disturbing data points is not fixed, thus ensuring differential privacy. 
\begin{figure*}[htbp]
	\centering
	\includegraphics[width=1\textwidth]{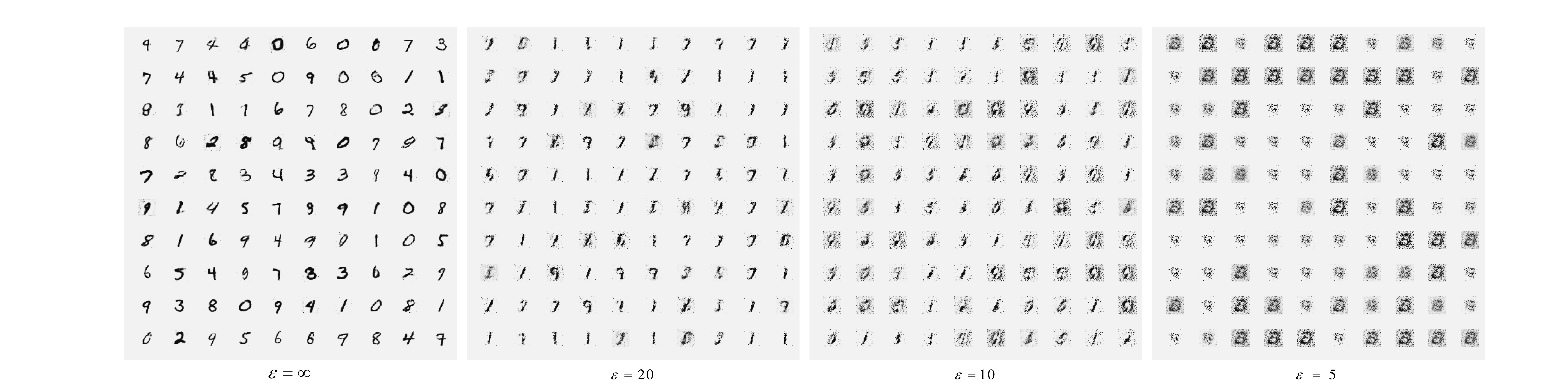}
	\caption{Four different $\epsilon$ values are generated for four different quality pictures on MNIST dataset.($\varepsilon  = \infty ,\varepsilon  = 20,\varepsilon  = 10,\varepsilon  = 5$; $\delta  = 1.0 \times {10^{ - 5}}$)  }\label{fig4}
\end{figure*}
\begin{figure}[htbp!]
\centering
	\includegraphics[width=1\textwidth]{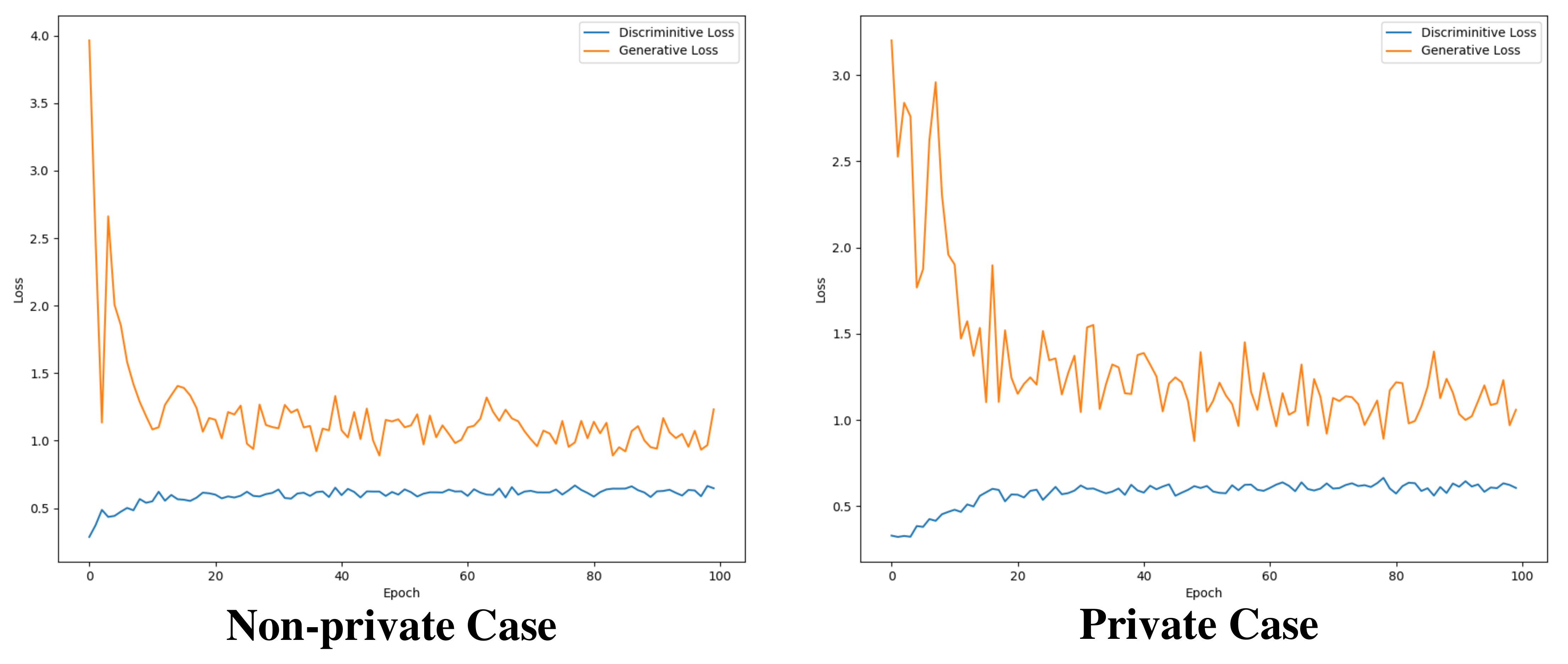}
	\caption{Loss of Non-private Case ($\varepsilon  = \infty $) and Private Case ($\varepsilon  \ne \infty $).}\label{fig5}
\end{figure}

Next, we will focus on the impact of noise on PPGAN's loss function during training. The results are shown in Fig. \ref{fig5} In the non-private case, we observe the training loss of the first 100 epoch in training. The result indicates that the loss of GAN is smooth and stable, and no large fluctuations exist in this round of training. When the loss of the PPGAN with noise starts to fluctuate at the tail of the curve, PPGAN can still converge. As can be inspected from Fig. \ref{fig5}, the convergence rate of PPGAN is acceptable as the compromise of the introduced privacy preservation capability.

\subsection{Relationship between Privacy budget and High-quality Datasets}
In this section, we quantitatively evaluate the performance of PPGAN. Specifically, we first compare generated data with real data based on statistical characteristics. We propose a \textit{Generate score} to measure the quality of data generated by GAN. We proposed \textit{Generate score} ($GS({P_g})$) to measure the quality of data generated by PPGAN, which can be formally defined as follows for $P_g$:
\begin{definition}(Generate scores):\label{defi-7}
	\begin{equation}	
	\begin{array}{l}
	IS({P_g}) = {e^{{E_{x \sim {P_g}}}[KL(PM(y|x)||PM(y))]}}\\
	GS({P_g}) = |\frac{{IS({P_g}) - mean(IS({P_g}))}}{{\max (IS(Pg)) - \min (IS({P_g}))}}|
	\end{array}
	\end{equation}
	where $IS({P_g})$ is \textit{Inception score} which  is measure of the performance of the GAN.
\end{definition}
\begin{figure}[htbp]\centering
	\includegraphics[width=0.5\textwidth]{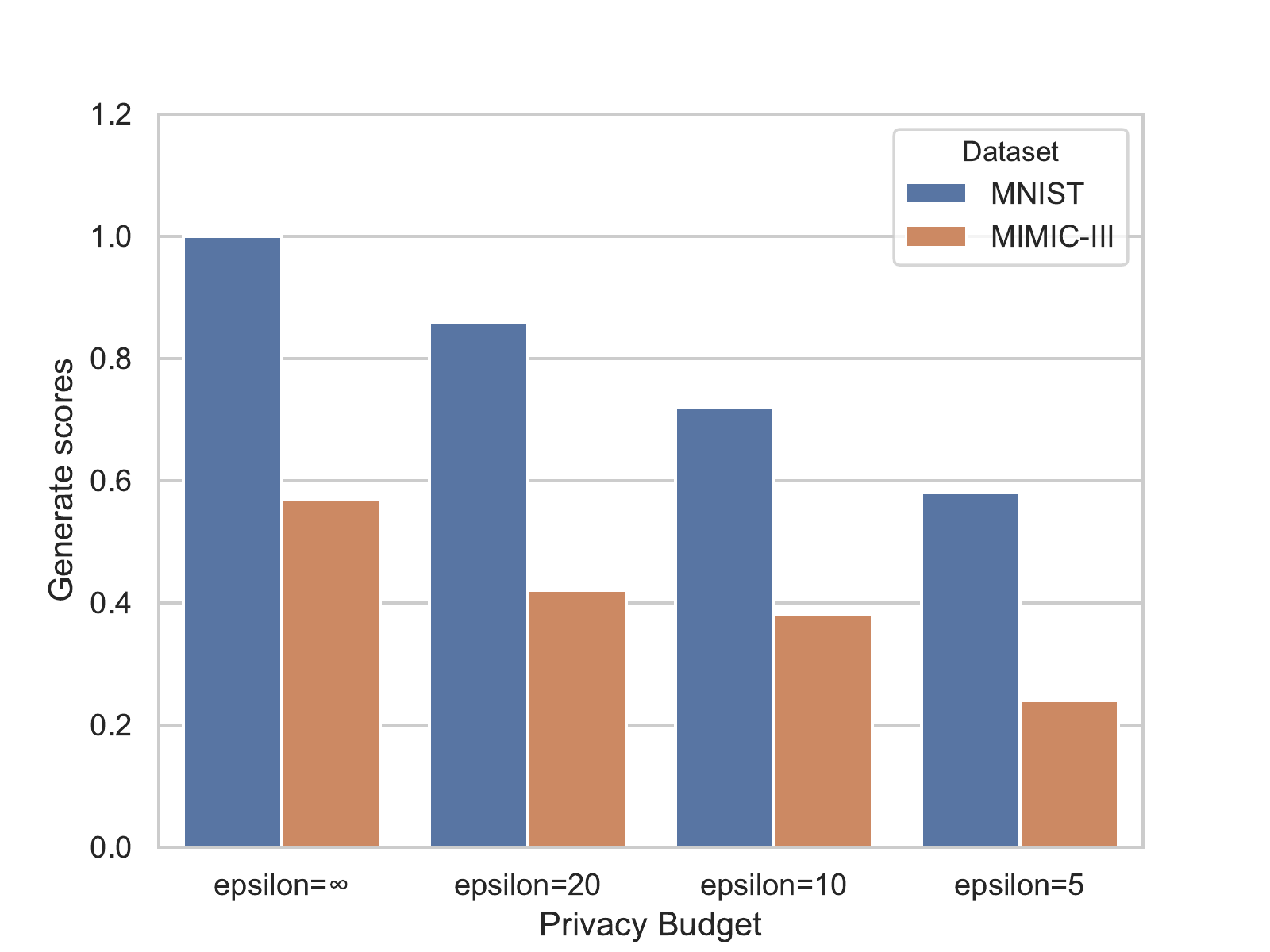}
	\caption{Generate scores of generative data on MNIST.}\label{fig6}
\end{figure}

The experimental result is shown in Fig. \ref{fig6}. The generated data's (generated by PPGAN) generate score is compared to the real data of the MNIST dataset with different privacy budgets. The larger the score value, the better the quality of the data generated by the generator. The figure shows the distribution of the generate scores of PPGAN in the case of $\epsilon=20,10,5$. It can be seen from the figure that the score is very close to the real data generated by the WGAN (non-private case, $ \varepsilon = \infty $.). When $\epsilon=20$, the PPGAN generate score is only 0.14 different from the WGAN generate score, which indicates that the PPGAN generation quality is close to the WGAN.

To evaluate the performance of PPGAN, we compare three solutions, namely dp-GAN \cite{ref-3}, DPGAN \cite{ref-18} and WGAN \cite{ref-20} (Non-private Case) in terms of the quality of the generated data.
\begin{figure}[htbp]\centering
	\includegraphics[width=0.5\textwidth]{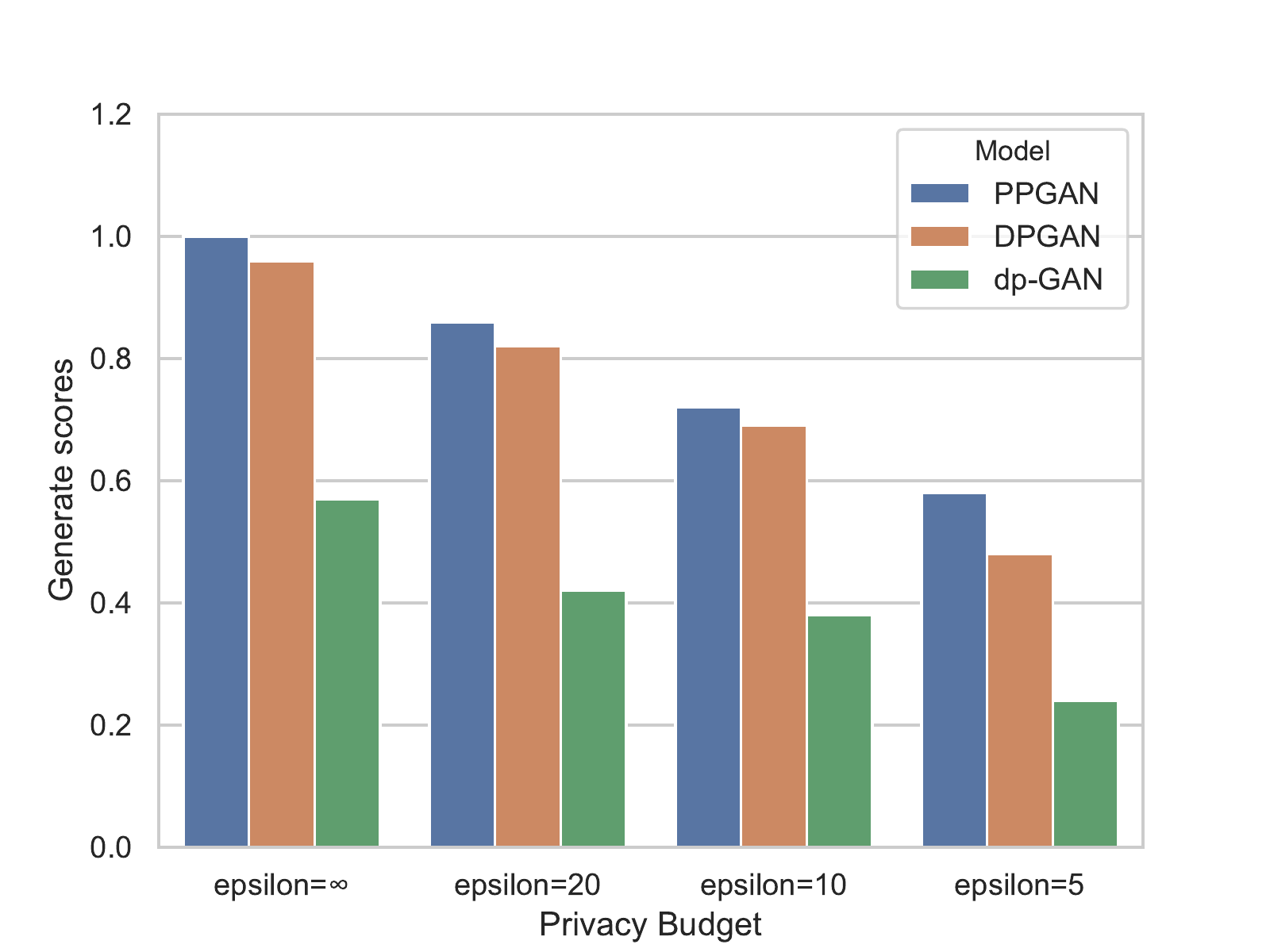}
	\caption{Generate scores of generative data on model PPGAN, DPGAN and dp-GAN. ($\delta  = 1.0 \times {10^{ - 5}}$)}\label{fig7}
\end{figure}

As can be seen from Fig. \ref{fig7}, the data quality generated by PPGAN is better than dp-GAN and DPGAN.

%
\section{CONCLUSION}\label{sec-5}
In this paper, we propose the PPGAN model that preserves the privacy of training data in a differentially private case. PPGAN mitigates information leakage by adding well-designed noise to the gradient during the learning process. We conducted two experiments to show that the proposed algorithm can converge under the noise and constraints of the training data and generate high-quality data. Also, our experimental results verify that PPGAN does not suffer from mode collapse or gradient disappearance during training, thus maintaining excellent stability and scalability of model training.

\section*{ACKNOWLEDGMENTS}
This work is supported by the Ministry of Education of China and the School of Entrepreneurship Education of Heilongjiang University (Grant NO.201910212133) and Heilongjiang Provincial Natural Science Foundation of China (Grant NO.QC2016091).

\bibliographystyle{unsrt}  

\bibliographystyle{IEEEtran}
\bibliography{reference}

\end{document}